\documentclass[conference]{IEEEtran}
\IEEEoverridecommandlockouts
\usepackage{cite}
\usepackage{amsmath,amssymb,amsfonts,amsthm}
\usepackage{algorithm}
\usepackage{algpseudocode}
\usepackage{graphicx}
\usepackage{textcomp}
\usepackage{xcolor}
\usepackage{url}

\newcommand{\R}{\mathbb{R}}

\newcommand{\N}{\mathbb{N}}

\newcommand{\Prob}{\mathbb{P}}
\newcommand{\Exp}{\mathbb{E}}

\newtheorem{definition}{Definition}[]
\newtheorem{proposition}{Proposition}[]
\def\BibTeX{{\rm B\kern-.05em{\sc i\kern-.025em b}\kern-.08em
    T\kern-.1667em\lower.7ex\hbox{E}\kern-.125emX}}
\begin{document}

\title{Fast Risk Assessment for Autonomous Vehicles Using Learned Models of Agent Futures}

\newcommand{\CH}[1]{\textcolor{red}{[CH: #1]}}
\newcommand{\AW}[1]{\textcolor{red}{[AW: #1]}}
\newcommand{\AJ}[1]{\textcolor{red}{[AJ: #1]}}
\let\oldemptyset\emptyset
\let\emptyset\varnothing
\author{Allen Wang%
, Xin Huang%
, Ashkan Jasour%
, and Brian C. Williams \thanks{All authors are with the Computer Science and Artificial Intelligence Laboratory (CSAIL), Massachusetts Institute of Technology
        {\tt\small \{allenw, huangxin, jasour, williams\} @ mit.edu}}}

\maketitle

\begin{abstract}
This paper presents fast non-sampling based methods to assess the risk of trajectories for autonomous vehicles when probabilistic predictions of other agents' futures are generated by deep neural networks (DNNs). The presented methods address a wide range of representations for uncertain predictions including both Gaussian and non-Gaussian mixture models for predictions of both agent positions and controls. We show that the problem of risk assessment when Gaussian mixture models (GMMs) of agent positions are learned can be solved rapidly to arbitrary levels of accuracy with existing numerical methods. To address the problem of risk assessment for non-Gaussian mixture models of agent position, we propose finding upper bounds on risk using Chebyshev's Inequality and sums-of-squares (SOS) programming; they are both of interest as the former is much faster while the latter can be arbitrarily tight. These approaches only require statistical moments of agent positions to determine upper bounds on risk. To perform risk assessment when models are learned for agent controls as opposed to positions, we develop \textit{TreeRing}, an algorithm analogous to tree search over the ring of polynomials that can be used to exactly propagate moments of control distributions into position distributions through nonlinear dynamics. The presented methods are demonstrated on realistic predictions from DNNs trained on the Argoverse and CARLA datasets and are shown to be effective for rapidly assessing the probability of low probability events.
\end{abstract}
\section{Introduction}
In order for autonomous vehicles to drive safely on public roads, they need to predict the future states of other agents (e.g. human-driven vehicles, pedestrians, cyclists) and plan accordingly. Predictions, however, are inherently uncertain, so it is desirable to represent uncertainty in predictions of possible future states and reason about this uncertainty while planning. This desire is motivating ongoing work in the behavior prediction community to go beyond single mean average precision (MAP) prediction and develop methods for generating probabilistic predictions \cite{chai2019multipath, rhinehart2018r2p2, lee2017desire,huang2019diversity}. In the most general sense, this involves learning joint distributions for the future states of all the agents conditioned on their past trajectories and other context specific variables (e.g. an agent is at a stop light, lane geometry, the presence of pedestrians, etc). However, learning such a distribution can often be intractable, so current works use a wide variety of different simplified representations for probabilistic predictions. \cite{lee2017desire} trains a conditional Variational Autoencoder (CVAE) to generate samples of possible future trajectories. Other works use generative adversarial networks (GANs) to generate multiple trajectories with probabilities assigned to each of them \cite{huang2019diversity, li2019interaction}. As a discrete alternative, \cite{hong2019rules,bansal2018chauffeurnet} train a DNN to generate a probabilistic occupancy grid map with a probability assigned to each cell. However, such grid-based approaches effectively treat possible agents' trajectories as belonging to a discrete space, while, in reality, agents may be at an uncountable number of points in continuous space. Many recent papers try to account for the continuous nature of uncertainty in space by learning (GMMs) for vehicle positions \cite{chai2019multipath, deo2018multi, hong2019rules} or coefficients of polynomials in $\R^2$ that represent the vehicles' positions \cite{huang2019uncertainty}. Since learning uncertain models for position or pose can also sometimes produce results that are inconsistent with basic kinematics, some recent works develop DNNs that predict future control inputs which are then propagated through a kinematic model to predict future positions \cite{cui2019deep, rhinehart2018r2p2}.

Given a probabilistic prediction, an autonomous vehicle still needs to be able to \textit{rapidly} evaluate the probability of a given plan resulting in a collision or, more generally, a constraint violation. We will refer to this problem as \textit{risk assessment} and it is particularly challenging in the context of autonomous driving as 1) autonomous vehicles need to reason about low probability events to be safer than human drivers and 2) there are hard real time constraints on algorithm latency. Latency is a critical consideration for safety and will be a major consideration motivating the methods presented in this paper. While an algorithm with a latency of, for example, one second would often be acceptable in other robotics applications, it would be unacceptable for an autonomous vehicle traveling at 20 m/s on public roads. This requirement of low latency while retaining the ability to reason about low probability events makes naive Monte Carlo computationally intractable. To address this problem, adaptive and importance sampling methods have been proposed to estimate these probabilities with fewer samples \cite{schmerling2016evaluating, norden2019efficient}. However, such methods still do not usually simultaneously provide guarantees on both latency and error. Their performance can also be highly sensitive to algorithm parameters and proposal distributions.

\indent\textit{Statement of Contributions}: We present fast methods to assess the risk of trajectories for both Gaussian and non-Gaussian position and control models of other agents. In section \ref{sec:risk_assess}, we begin by addressing the case when GMMs are used for agent position predictions. We show this particular case can be reduced to the problem of computing the CDF of a quadratic form in a multivariate Gaussian (QFMVG)- a well-studied problem in the statistics community for which methods exist that can rapidly solve it to arbitrary accuracy. To address the more general case when potentially non-Gaussian mixture models are used for agent position predictions, we apply statistical moment-based approaches to determine upper bounds on risk, which we will refer to as \textit{risk bounds}. Namely, we propose using Chebyshev's Inequality and a sums-of-squares (SOS) program that can be seen as a generalization of Chebyshev's Inequality; the former is faster, while the latter can provide arbitrarily tight risk bounds. These moment-based approaches have the feature of being \textit{distributionally robust}, producing risk bounds that are true for all possible distributions that take on the value of the given moments.
To address uncertain models for controls, in Section \ref{sec:moment_prop}, we develop \textit{TreeRing}, a novel algorithm analogous to tree search, but over the ring of polynomials. Given a polynomial stochastic system, \textit{TreeRing} can find polynomial expressions for moments of position in terms of moments of control random variables. This enables the application of our non-Gaussian position risk assessment methods to the problem of risk assessment when models are learned for agent controls. Figure \ref{fig:control_framework_illustration} illustrates our framework in this case. In Section \ref{sec:experiments}, we demonstrate our methods on realistic predictions generated by DNNs trained on the Argoverse and CARLA datasets \cite{chang2019argoverse, Dosovitskiy17}. Source code can be found at \url{https://github.com/allen-adastra/risk_assess}.
\begin{figure}[!b]
    \vspace*{-5mm}
    \centering
    \includegraphics[width=0.8\linewidth]{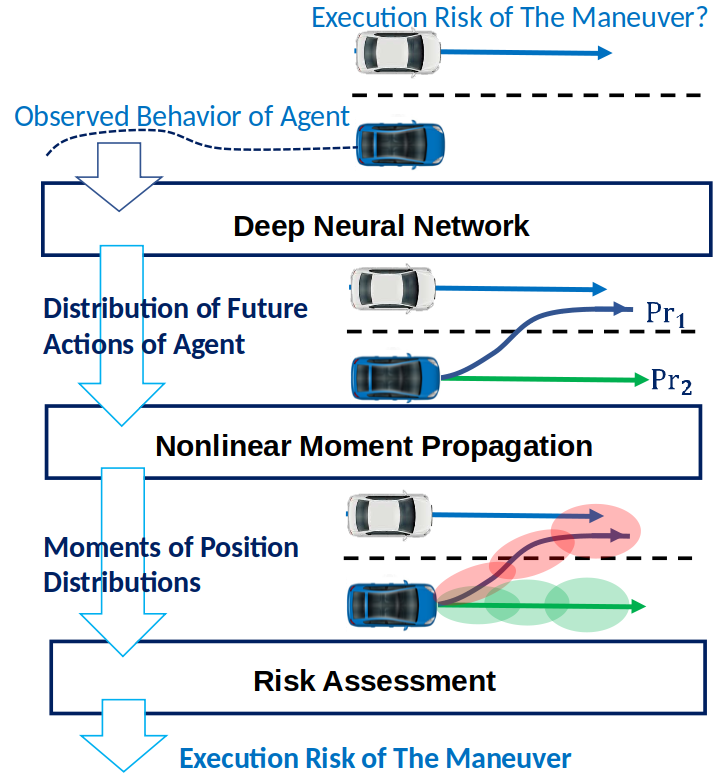}
    \caption{Illustration of our risk assessment framework when control distributions are used for predictions. When position distributions are used, the nonlinear moment propagation step is skipped.}
    \label{fig:control_framework_illustration}
\end{figure}

\section{Notation}
Let $\mathcal{S}^n_{++}$ denote the set of $n\times n$ positive definite matrices. For any matrix $Q\in\mathcal{S}^n_{++}$ and vector $\mathbf{x}\in\R^n$, let $Q(\mathbf{x}):=\mathbf{x}^TQ\mathbf{x}$. Let $Q_{ij}$ denote the element in the $i_{th}$ row and $j_{th}$ column of $Q$. For any $\theta\in\R$, let $R(\theta)$ be the 2D rotation matrix parameterized by $\theta$. For a vector $\mathbf{x}\in\R^n$ and multi-index $\alpha\in\mathbb{Z}^n_+$, let $\mathbf{x}^\alpha = \prod_{i=1}^nx_i^{\alpha_i}$. Let $\R[\mathbf{x}]$ denote the ring of polynomials in $\mathbf{x}$ over $\R$. For $n\in\N$, let $[n] = \{k\in\N : k\leq n\}$. A polynomial $s(x)\in\R[x]$ is said to be sums-of-squares (SOS) if for some $l\in\N$, $\exists h_i(x)\in\R[x]$ for $i\in [l]$ s.t. $s(x) = \sum_{i=1}^l h_i(x)^2$. For a random vector $\mathbf{w}$ and any $d\in\N$, let $\mu_{\mathbf{w}_t}, \Sigma_{\mathbf{w}_t}$ denote its mean vector and covariance matrix respectively, and $\Phi_\mathbf{w}$ denote its characteristic function. For any set $S$, let $\mathcal{P}(S)$ denote the power set of $S$ with the empty set removed, $|S|$ denote the cardinality of $S$, and $S^n$ denote the $n$\textit{-ary} Cartesian power. For a vector valued function $f$, $f_i$ denotes the $i_{th}$ component of $f$.
\section{Problem Statement}
We define \textit{risk} as the probability of an agent entering an ellipsoid around the ego vehicle. Thus, we are interested in computing the probability of other agents entering:
\begin{align}
    \left\{\mathbf{x}\in\R^2 : Q(\mathbf{x})\leq 1\right\}, \quad Q\in\mathcal{S}_{++}^2
\end{align}
We argue ellipsoids are a useful representation as: 1) they can be fit relatively tightly to the profiles of vehicles, and 2) the sizes of both the ego vehicle and agent can be accounted for by properly scaling the size of the ellipsoid around the vehicle. Throughout the paper, agent positions at each time step are always defined in the frame of the planned future poses of the ego vehicle unless stated otherwise; Section \ref{sec:change_frame} shows how moments of distributions can be expressed in different frames. Given this formulation, the ellipsoid is parameterized by a constant matrix $Q\in\mathcal{S}^2_{++}$ in the ego vehicle frame. In practice, multiple ellipsoids can be defined around the vehicle and an appropriate one selected at run-time. We restrict our focus to the single agent case and note that the risk in a multi-agent setting can be upper bounded by summing the risk associated with each agent.

If $\mathbf{x}_t = [x_t, y_t]^T$ is some random vector for the position of the agent at time $t$, then the risk associated with an agent across the whole $T$ step time horizon is:
\begin{align}
    \mathcal{R} &:= \Prob\left(\bigcup_{t=1}^{T}\left\{Q(\mathbf{x}_t)\leq 1\right\}\right)\label{eq:agent_risk}
\end{align}
 By the inclusion-exclusion principle, the probability $(\ref{eq:agent_risk})$ can be computed as the sum of the probabilities of the marginal events and the probabilities of all possible intersections of events:
\begin{align}
\mathcal{R} = \sum_{J\in\mathcal{P}([T])} (-1)^{|J|+1}\Prob\left(\bigcap_{j\in J} \{Q(\mathbf{x}_{j})\leq 1\}\right)
\end{align}
In many works, the random variables are assumed to be independent across time or can be made to be independent across time by conditioning on a discrete mode \cite{chai2019multipath,deo2018multi,hong2019rules}. If there is dependence across time, one would need the conditional distributions of the events which require additional information to be learned. As most work on behavior prediction currently assumes independence across time, this paper restricts its focus to the time independent case, and so:
\begin{align}
    \mathcal{R} = 1 - \prod_{t\in[T]} \left(1 - \Prob(Q(\mathbf{x}_t)\leq 1)\right)
\end{align}
Thus, the problem of risk assessment along the trajectory can be solved by computing the marginals at each time step $t$, so the rest of the paper restricts its focus to the marginals.
\section{Risk Assessment}\label{sec:risk_assess}
In this section, we present solutions for both Gaussian and non-Gaussian risk assessment when moments of the random vector for agent position $\mathbf{x}_t$ are known. We begin by addressing the problem of determining moments of agent positions in different frames to account for the ego vehicles planned trajectory. We then present our solution for the GMM case using numerical approximations of the CDFs of QFMVGs. To address the non-Gaussian case, we present methods based off Chebyshev's Inequality and SOS programming. We assume basic knowledge of SOS programming;
We refer the reader to \cite{parrilo2003semidefinite,rarnop} for an overview of SOS programming and \cite{jasour2019RiskMap,jasour2019Tube,jasour2018moment,jasour2016PhD,jasour2015SIAM,rarnop} for moment-SOS based planning.
Throughout this section, we assume the necessary moments of $\mathbf{x}_t$ are known.

\subsection{Changing Frames}\label{sec:change_frame}
Predictions are usually given in a global frame, so this section provides a method for transforming the global frame distribution moments into the ego vehicle frame. More generally, we are concerned with computing moments of $\mathbf{x}_t$ in a new frame offset by $\mathbf{v}\in\R^2$ and rotated by $-\theta\in\R$. As shown in the appendix, if $\mathbf{x}_t$ is a mixture model, its moments can be computed in terms of moments of its components, so, in this section, let $\mathbf{x}_t$ be a component of a mixture model. We propose only translating the moments and then accounting for the rotation by using $Q^* = R(\theta)^TQR(\theta)$ instead of $Q$. The rotation can be accounted for by using $Q^*$ instead of $Q$ because:
\begin{align}
    \mathbf{x}_t^TQ^*\mathbf{x}_t &= \mathbf{x}_t^TR(\theta)^TQR(\theta)\mathbf{x}_t\\&= (R(\theta)\mathbf{x}_t)^TQ(R(\theta)\mathbf{x}_t)
\end{align}
The translated moments can be computed by applying the binomial theorem to $(\mathbf{x}_t - \mathbf{v})^n$ (here, the power is applied element-wise). Note that applying the binomial theorem to $(\mathbf{x}_t - \mathbf{v})^n$ requires moments of $\mathbf{x}_t$ up to order $n$.

\subsection{Risk Assessment for GMM Position Models}
In this section, we provide a method to solve the risk assessment problem when the uncertain prediction is represented as a sequence of GMMs, $\mathbf{x}_t$, of the agents position with discrete modes determined by the Multinoulli $Z_t$. Many works currently learn GMMs for vehicle position as they express both multi-modal and continuous uncertainty \cite{chai2019multipath, deo2018multi, hong2019rules}. As shown in Figure~\ref{fig:risk_assess_example}, they provide an intuitive representation of uncertainty in both the drivers high level decisions and low level execution.
With time independence, the risk is:
\begin{align}
    \sum_{z = 1}^n\left(1 - \prod_{t\in[T]} 1 - \Prob(Q(\mathbf{x}_t)\leq 1 : Z_t=z)\right)\Prob(Z_t = z)
\end{align}
\begin{figure}[!t]
    \centering
    \includegraphics[width=0.95\linewidth]{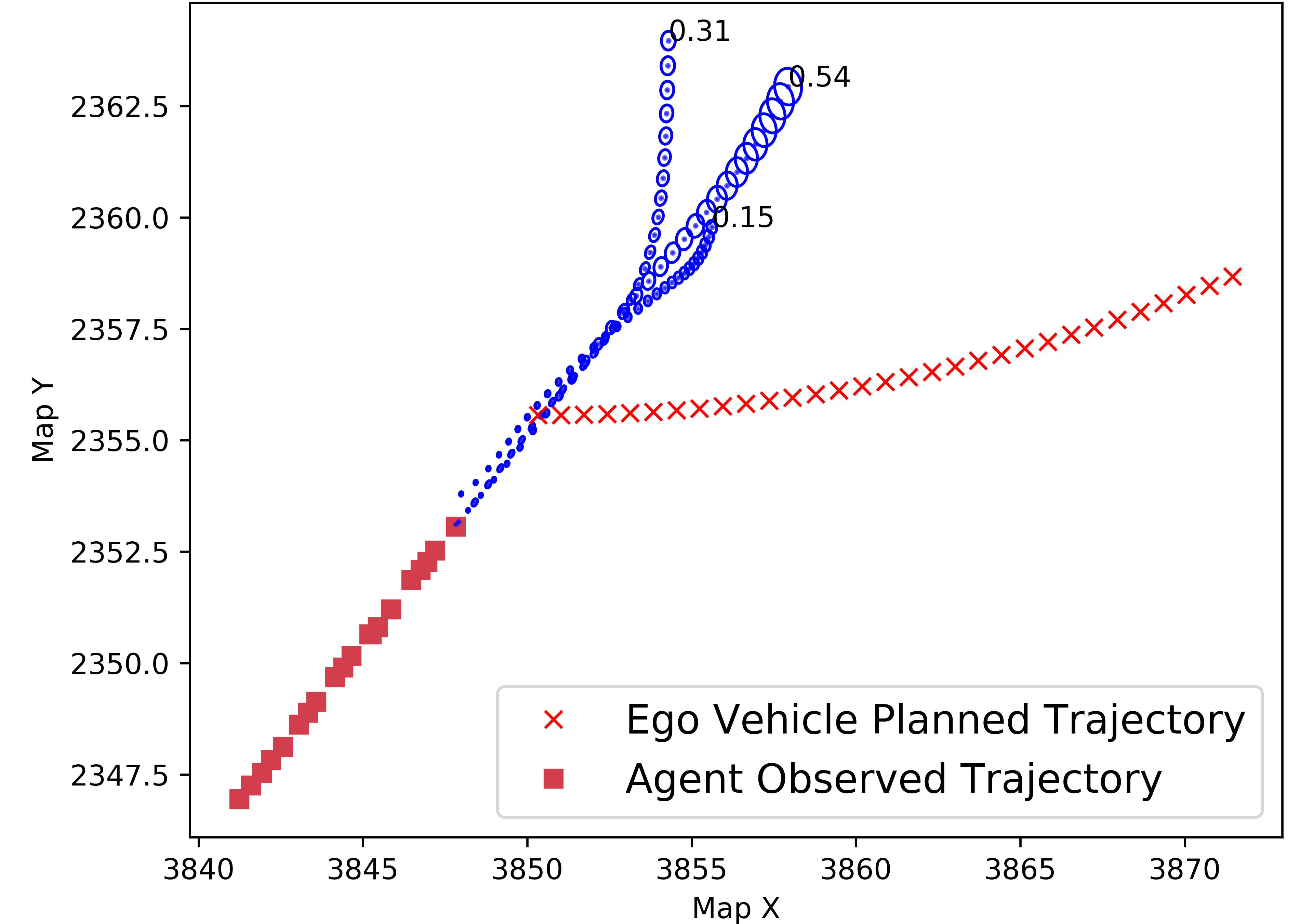}
    \caption{An example risk assessment scenario. One standard deviation confidence ellipses (in blue) of a multi-modal GMM prediction are shown with mode probabilities. The observed agent trajectory and planned ego vehicle trajectory are also shown in red with different markers.}
    \label{fig:risk_assess_example}
    \vspace*{-5mm}
\end{figure}
Note that the above expression can be easily modified for the case when there is a single Multinoulli random variable that is constant across all time, an assumption used in, for example, \cite{chai2019multipath}. The probabilities $\Prob(Z_t = z)$ are learned parameters of GMMs, so the problem of risk assessment can be solved by computing $\Prob\left(Q(\mathbf{x}_t)\leq 1 : Z_t = z\right)$ for each agent, time step, and mode. Note that this is exactly the CDF of $Q(\mathbf{x}_t)$ conditioned on $Z_t = z$ which is a quadratic form in a multivariate Gaussian (QFMVG). Unfortunately, there does not exist a known closed form solution to exactly evaluate the CDF of QFMVGs, but fast approximation methods with bounded errors have been studied within the statistics community \cite{liu2009new, solomon1977distribution, kotz1967series, duchesne2010computing, imhof1961computing}. Several of these methods have been implemented in the R package CompQuadForm \cite{de2017compquadform}. Of particular interest is the method of Imhof, which produces results with bounded approximation error by numerical inversion of the characteristic function of the QFMVG \cite{imhof1961computing}. A faster, but less accurate, alternative is the method of Liu-Tang-Zhang which involves approximating the CDF of the QFMVG with the CDF of a non-central chi square distribution with parameters chosen to minimize the difference in kurtosis and skew between the approximate and target distributions \cite{liu2009new}.
\subsection{Non-Gaussian Risk Assessment with Chebyshevs Inequality}
As a consequence of the one-tailed Chebyshev's Inequality, for any measurable function $g$, whenever $\mathbb{E}[g(\mathbf{x}_t)] > 0$, we have that:
\begin{align}
    \Prob(g(\mathbf{x}_t)\leq 0)&\leq \frac{\Exp[g(\mathbf{x}_t)^2] - \Exp[g(\mathbf{x}_t)]^2}{\Exp[g(\mathbf{x}_t)^2]}
\end{align}
That is, the first two moments of $g(\mathbf{x}_t)$ are sufficient to establish a bound on the risk that the constraint $g(\mathbf{x}_t)\leq 0$ is violated. We note that the requirement $\mathbb{E}[g(\mathbf{x}_t)] > 0$ is not particularly restrictive because $\mathbb{E}[g(\mathbf{x}_t)]\leq 0$ means the average case involves collision, thus corresponding to what is usually an unacceptable level of risk.
\subsubsection{Applying Chebyshev's Inequality to the Quadratic Form}\label{chebyshev_quad_form}
To apply Chebyshev's inequality to $\Prob(Q(\mathbf{x}_t) - 1\leq 0)$, we would need the first two moments of $Q(\mathbf{x}_t)-1$ which can be expressed in terms of the first two moments of $Q(\mathbf{x}_t)$. The first moment can be expressed in terms of the mean vector and covariance matrix of $\mathbf{x}_t$ \cite{provost1992quadratic}:
\begin{align}
    \mathbb{E}[Q(\mathbf{x}_t)] = \text{Tr}(Q\Sigma_{\mathbf{x}_t}) + \mu_{\mathbf{x}_t}^TQ\mu_{\mathbf{x}_t}
\end{align}
We can determine an expression for $\mathbb{E}[(Q(\mathbf{x}_t)^2]$ via an alternate representation for the quadratic form:
\begin{align}
    \mathbb{E}[Q(\mathbf{x}_t)^2] &= \sum_{(i,j,k,l)\in [2]^4}Q_{ij}Q_{kl}\mathbb{E}\left[x_{t_i}x_{t_j}x_{t_k}x_{t_l}\right]
\end{align}
Thus, to compute the second moment of $Q(\mathbf{x}_t)$, we would need the moments of $\mathbf{x}_t$ of order up to four.
\subsubsection{Conservative Approximation with Half-Spaces}
It's possible to reduce the order of the moments that need to be propagated to two by instead approximating the ellipsoid as the intersection of $n_h$ half-spaces parameterized by $\mathbf{a}_i\in\R^2$ and $b_i\in\R$. The approximated set is thus:
\begin{align}
\mathcal{X}_{Approx} = \cap_{i=1}^{n_h}\{\mathbf{x}\in\mathbb{R}^2 : \mathbf{a}_i^T\mathbf{x} + b_i\leq 0\}
\end{align}
Since the probability of any individual event is greater than the probability of the intersection of events, we have that:
\begin{align}
    \Prob(\cap_{i=1}^{n_h} \{\mathbf{a}_i^T\mathbf{x}_t + b_i\leq 0\})\leq \min_{i\in[n_h]} \Prob(\mathbf{a}_i^T\mathbf{x}_t + b_i \leq 0)
\end{align}
So if we determine an upper bound on the probability of each $\Prob(\mathbf{a}_i^T\mathbf{x}_t + b_i \leq 0)$ with Chebyshev's Inequality, the minimum of the Chebyshev bounds will be an upper bound on our risk. Since $\mathbf{a}_i^T\mathbf{x}_t + b_i$ is an affine transformation of $\mathbf{x}_t$, its mean and variance can be expressed with the mean vector and covariance matrix of $\mathbf{x}_t$.
\subsection{Non-Gaussian Risk Assessment with SOS Programming}\label{sos_subsection}
When tighter risk bounds are desired than those obtained via Chebyshev's Inequality, for any measurable function $g$,  an \textit{univariate} SOS program can be used to upper-bound $\Prob(g(\mathbf{x}_t)\leq 0)$ -- the SOS program is univariate in the sense that it searches for a polynomial in a single indeterminant, not in the sense that there is only one decision variable \cite{jasour2018moment}. The fact that the SOS program is univariate is significant because the key disadvantages of SOS, scalability and conservatism, are not as limiting for univariate SOS because: 1) the number of decision variables in the resulting SDP scales quadratically w.r.t. the order of the polynomial we are searching for and 2) the set of non-negative univariate polynomials is equivalent to the set of univariate SOS polynomials, allowing univariate SOS to explore the full space of possible solutions.

We begin by noting that the probability of constraint violation is equivalent to the expectation of the indicator function on the sub-level set of $g$:
\begin{align}
    \Prob(g(\mathbf{x}_t)\leq 0) = \mathbb{E}[\mathbf{1}_{g(\mathbf{x}_t)\leq 0}]
\end{align}
The expectation of the indicator function, however, is not necessarily easily computable. To solve this problem, we find some polynomial with a more easily computable expectation that upper bounds the indicator function. If we can find some univariate polynomial, $p: \R\rightarrow\R$ of order $d$ in some indeterminant $x\in\R$ with coefficients $c_k, k = 0,...,d$ that upper bounds the indicator function, then clearly the following implication holds by substitution:
\begin{align}
    p(x):= \sum_{k=0}^d c_k x^k\geq \mathbf{1}_{x\leq 0}\Rightarrow \sum_{k=0}^d c_kg(\mathbf{x}_t)^k\geq\mathbf{1}_{g(\mathbf{x}_t)\leq 0}
\end{align}
Given the coefficients $c_k$, if we apply the expectation w.r.t. the density function of $\mathbf{x}_t$ to both sides, then we can reduce the problem of finding an upper bound on $\Prob(g(\mathbf{x}_t)\leq 0)$ to that of computing moments of the random variable $g(\mathbf{x}_t)$:
\begin{align}
    \sum_{k=0}^dc_k\mathbb{E}[g(\mathbf{x}_t)^k]\geq \mathbb{E}[\mathbf{1}_{g(\mathbf{x}_t)\leq 0}] = \Prob(g(\mathbf{x}_t)\leq 0)
\end{align}
The moments of $g(\mathbf{x}_t)$, in turn, are computable in terms of moments of $\mathbf{x}_t$ by expanding out the polynomial power and applying the linearity of expectation. For example, if $g(\mathbf{x}_t) = x_t^2 + y_t^2$, then:
\begin{align}
    \mathbb{E}[g(\mathbf{x}_t)^3] = \mathbb{E}[x_t^6] + 3\mathbb{E}[x_t^4y_t^2] + 3\mathbb{E}[x_t^2y_t^4] + \mathbb{E}[y_t^6]
\end{align}
In the case that $\mathbf{x}_t$ is a multivariate Gaussian, the higher order central moments and central cross moments can be computed in close form given the mean vector and covariance matrix. In the non-Gaussian case, there may not be a convenient way to determine the desired moments, but we note that the moments can often be computed using moment generating functions or sampling based approaches. In this section, we assume that we know the necessary moments of $\mathbf{x}_t$ to compute $\mathbb{E}[g(\mathbf{x}_t)^k], \forall k\in[d]$. We also normalize the moments, as doing so improves the numerical conditioning of the problem \footnote{normalization is valid because $\Prob(X\leq 0) = \Prob(cX\leq 0)$ for $c>0$}.

Now consider the following univariate SOS program in the indeterminant $x$ which can search for the polynomial which minimizes the upper bound on risk.
\begin{subequations}
\begin{alignat}{2}
    \min_{p, s_1, s_2}\quad &\sum_{k=0}^d c_k\mathbb{E}[g(\mathbf{x}_t)^k]\label{opt_obj}\\
    & p(x) - 1 = s_1(x) - xs_2(x) \label{non_neg_on_negative}\\
    & p(x), s_1(x), s_2(x) \quad SOS
\end{alignat}
\end{subequations}
If the order of the polynomial is chosen to be $d = 2n$ for some $n\in\N$, then we should have that $\text{deg}(s_1) = d$ and $\text{deg}(s_2) = d - 2$. If $d = 2n + 1$ for some $n\in\N$, then we should have that $\text{deg}(s_1) = 2n$ and $\text{deg}(s_2) = 2n$. Note that the constraint (\ref{non_neg_on_negative}) enforces:
\begin{align}
    p(x) \geq 1\quad \forall x\in(-\infty, 0] \label{noneg_on_negative_condition}
\end{align}
And since $p(x)$ is constrained to be SOS, it is also globally non-negative so $p(x)\geq\mathbf{1}_{x\leq 0}, \forall x\in\R$.
Thus, we can see that the optimal objective value of this SOS program yields an upper bound on $\Prob(g(\mathbf{x}_t)\leq 0)$.

\section{Moment Propagation}\label{sec:moment_prop}
While directly learning distributions for agents future positions can be an effective strategy, one major disadvantage is it can produce physically unrealistic predictions. \cite{cui2019deep, rhinehart2018r2p2} address this by learning distributions for control inputs and then propagating samples through a kinematic model. While the Kalman filter and its variants, such as the extended and unscented Kalman filters, can be used to propagate mean and covariance, they are not exact and do not immediately apply to higher order moments \cite{wan2000unscented,kalman1961new,julier2004unscented}.

In this section, we provide an approach for nonlinear moment propagation that can, in principle, work for moments up to arbitrary order \cite{jasour2016PhD}. Given a nonlinear model and a random vector for control, $\mathbf{w}_t$, this section is concerned with the problem of computing statistical moments of $\mathbf{x}_t$ s.t. the non-Gaussian risk assessment methods presented in Section \ref{sec:risk_assess} can be applied. More precisely, we are looking for moments of the form $\mathbb{E}[x_t^\alpha y_t^\beta]$ where $\alpha,\beta\in\N$. The only requirements we impose on $\mathbf{w}_t$ is that its entries have 1) bounded moments and 2) computable characteristic functions. In the case that mixture models are used, we show in the appendix that it is sufficient for the components of the mixture models to satisfy the requirements as the moments and characteristic functions of mixtures of random variables can be computed in terms of those of their components. We use a stochastic version of the discrete-time Dubin's car to both demonstrate the general approach and to address the problem of agent risk assessment:
\begin{subequations}
\begin{align}
    x_{t+1} &= x_t + v_t\cos(\theta_t)\\
    y_{t+1} &= y_t + v_t\sin(\theta_t)\\
    v_{t+1} &= v_t + w_{v_t}\\
    \theta_{t+1} &= \theta_t + w_{\theta_t}
\end{align}
\label{eq:dynamics}
\end{subequations}
Above, the control vector is $\mathbf{w}_t = [w_{v_t}, w_{\theta_t}]$ where $w_{v_t}$ and $w_{\theta_t}$ are random variables describing the agent's acceleration and steering at time $t$ and are assumed to be independent. $\mathbf{x}_t = [x_t, y_t]$ is the position of some reference point on the agent in a fixed frame, $v_t$ is its speed, and $\theta_t$ is the angle of its velocity vector with respect to the fixed frame. The time steps $\Delta t$ for discretization are omitted for brevity; the values of the variables can simply be scaled accordingly. 
\subsection{Motivating Example}
To motivate $\textit{TreeRing}$, we begin by showing how the dynamics of the moment $\mathbb{E}[x_{t+1}y_{t+1}]$ for the system $(\ref{eq:dynamics})$ can be found manually. $\textit{TreeRing}$ is essentially an automated version of this process. By substituting the equations $(\ref{eq:dynamics})$ in and applying the linearity of expectation, we arrive at the dynamics of our moment:
\begin{multline}
\mathbb{E}[x_{t+1}y_{t+1}] = \mathbb{E}[x_{t} y_{t}] + \mathbb{E}[v_{t}^{2} \sin{\left(\theta_{t} \right)} \cos{\left(\theta_{t} \right)}] \\+ \mathbb{E}[x_{t} v_{t} \sin{\left(\theta_{t} \right)}] + \mathbb{E}[y_{t}v_{t}  \cos{\left(\theta_{t} \right)}]
\end{multline}
Notice above that the term $\mathbb{E}[x_ty_t]$ shows up; if we find some way to compute the other three terms, we will arrive at an update relation with which we can compute $\mathbb{E}[x_ty_t]$ recursively. Turning our attention to the second term, we have that $v_t$ is independent of $\theta_t$ since we assumed $w_{v_t}$ and $w_{\theta_t}$ are independent at each time step. So the second term above can be factored by independence:
\begin{align}
    \mathbb{E}[v_t^2\sin(\theta_t)\cos(\theta_t)] = \mathbb{E}[v_t^2]\mathbb{E}[\sin(\theta_t)\cos(\theta_t)]
\end{align}
In the appendix, we show how the quantities $\mathbb{E}[v_t^2]$, which is the sum of independent random variables, and $\mathbb{E}[\sin(\theta_t)\cos(\theta_t)]$ can be computed in terms of the characteristic functions of $w_{v_t}$ and $w_{\theta_t}$, which we assumed are known. In the last two terms, there are dependencies between the variables, so there is not a clear way to factor them down into computable moments. Our main idea is to derive the dynamics of $\mathbb{E}[x_{t+1}v_{t+1}\sin(\theta_{t+1})]$ and $\mathbb{E}[y_{t+1}v_{t+1}\cos(\theta_{t+1})]$ by substituting in the equations (\ref{eq:dynamics}) much in the same way we did so for $\mathbb{E}[x_{t+1}y_{t+1}]$. We then simulate the dynamics of these additional moments in addition to $\mathbb{E}[x_{t+1}y_{t+1}]$; in a sense, we are producing a new dynamical system in terms of moments. In this case, recursively repeating this process of adding new moments to our dynamics produces a closed form set of equations that can recursively compute $\mathbb{E}[x_{t+1}y_{t+1}]$ in terms of moments of an initial state distribution and uncertainty at each time step. This process, however, is tedious and is easily subject to human error, especially for larger expressions. To address these issues, we developed $\textit{TreeRing}$ to algorithmically derive such expressions for moment propagation.
\subsection{TreeRing: An Algorithm for Moment Propagation}
\textit{TreeRing} is an algorithm for polynomial stochastic systems, but, in many cases, we are interested in a nonlinear system $f$ that is not polynomial, as is the case for (\ref{eq:dynamics}). However, nonlinear systems can always be approximated as a polynomial system by applying taylor expansions. In the case of (\ref{eq:dynamics}), the system can be made to be polynomial by applying the change of variables $c_t = \cos(\theta_t)$, $s_t = \sin(\theta_t)$, $c_{w_{t}} = \cos(w_t)$, and $s_{w_t} = \sin(w_{\theta_t})$:
\begin{subequations}
\begin{align}
    x_{t + 1} &= x_t + v_tc_t\\
    y_{t + 1} &= y_t + v_ts_t\\
    v_{t + 1} &= v_t + w_{v_t}\\
    c_{t + 1} &= c_tc_{w_t} - s_ts_{w_t}\\
    s_{t + 1} &= s_tc_{w_t} + c_ts_{w_t}
\end{align}\label{eq:dynamics_poly}
\end{subequations}
above, update relations for $c_t$ and $s_t$ were arrived at by using the trigonometric sums formulas to expand out $\cos(\theta_t + w_{\theta_t})$ and $\sin(\theta_t + w_{\theta_t})$ into the expressions shown.

Throughout this section, for any set of multi-indices, $\mathcal{Z}\subset\mathbb{N}^n$, let $\mathcal{Z}(\mathbf{b_t}) = \{\mathbb{E}[\mathbf{b_t}^\xi] : \xi\in\mathcal{Z}\}$ denote the set of corresponding moments. Given any $\xi\in\mathbb{N}^{n_b}$, suppose we want to be able to compute $\mathbb{E}[\mathbf{b}_t^\xi]$ for all $t$ in some finite horizon. The key idea is to find some set of multi-indices corresponding to moments $\mathcal{Z}\subset\mathbb{N}^{n_b}$ and some set of scalar-valued polynomials $\mathcal{F}\subset\R[\mathbf{b}_t]$ s.t. $\xi\in\mathcal{Z}$ and
$\forall m\in\mathcal{Z}(\mathbf{b}_{t+1})$, one of the following is true:
\begin{enumerate}
    \item The value of $m$ is known
    \item $\exists f\in\mathcal{F}$ s.t. $m = f\left(\mathcal{Z}(\mathbf{b}_t)\right)$
\end{enumerate}
That is, every moment at time $t + 1$ is either known or is a polynomial in the moments at time $t$. Thus, given $\mathcal{Z}(\mathbf{b}_t)$, we can directly compute $\mathcal{Z}(\mathbf{b}_{t + 1})$ with basic mathematical operations (addition, multiplication, exponentiation). By induction, if $\mathcal{Z}(\mathbf{b}_0)$ is known, as is the case when the moments of the initial distribution are known or the initial state is deterministic, we can compute all the moments for the entire time horizon. The key function of $\textit{TreeRing}$ is to search for the appropriate set of moments, $\mathcal{Z}$ and corresponding set of dynamics equations $\mathcal{F}$.

\indent An important property of the ring of polynomials is that it is closed under multiplication and addition. Thus, if we wanted to express the moment $\mathbb{E}[\mathbf{b}_{t+1}^\xi]$, for some $\xi\in\mathbb{N}^{n_b}$, in terms of quantities at time $t$, we have that $\exists p\in\mathbb{R}[\mathbf{b}_t]$ s.t:
\begin{align}
    \mathbb{E}[\mathbf{b}_{t+1}^\xi] = \mathbb{E}[p(\mathbf{b}_t)]
\end{align}
Letting $\mathcal{A}\subset\mathbb{N}^{n_b}$ denote the set of monomial exponents of the terms of $p$ and $C_\mathcal{A} = \{c_\alpha\in\R : \alpha\in\mathcal{A}\}$ denote the set of coefficients, we have by the linearity of expectation that:
\begin{align}
    \mathbb{E}[\mathbf{b}_{t+1}^\xi] = \sum_{\alpha\in\mathcal{A}} c_\alpha\mathbb{E}[\mathbf{b}_t^\alpha]\label{eq:update_expansion}
\end{align}
Both $\mathcal{A}$ and $C_\mathcal{A}$ can be found by using compute algebra packages such as SymPy \cite{meurer2017sympy}. So we already have that $\mathbb{E}[\mathbf{b}_{t+1}^\xi]$ can be expressed as a polynomial in moments at time $t$. To further reduce the problem, we factor each $\mathbb{E}[\mathbf{b}_t^\alpha]$ into the product of moments of smaller expressions (for example, this is clearly doable in the case when the variables in $\mathbf{b}_t$ are all independent of each other). To provide a computational method to search for these factorization, we introduce the following definitions which are used in Proposition \ref{prop:moment_to_moment}. Proposition $\ref{prop:moment_to_moment}$ essentially states that factorization can be found by simple search algorithms on the dependence graph. In TreeRing, once a decomposition has been found, moments of the form $\mathbf{b}_{t}^{\alpha[c]}$ that are not already accounted for in $\mathcal{Z}$ are simply added. The expand subroutine is then called with $\alpha[c]$ input in the first argument. Algorithm \ref{alg:create_update_relations} summarizes the algorithm.
\begin{definition}\normalfont
Let $V_\mathbf{b}$ be the set of variables in $\mathbf{b}$, $n_b = |V_\mathbf{b}|$, and $E_{\mathbf{b}}$ be the set of undirected edges s.t. $(b_i, b_j)\in E_{\mathbf{b}}$ i.f.f. $b_i$ and $b_j$ are dependent. Then:
\begin{itemize}
    \item The graph $G_{\mathbf{b}} = (V_\mathbf{b}, E_\mathbf{b})$ is the \textit{dependence graph} of $\mathbf{b}$.
    \item Given any multi-index $\alpha\in\mathbb{N}^{n_b}$, \textit{the subgraph of $G_\mathbf{b}$ induced by $\alpha$}, denote it $G_\mathbf{b}[\alpha] = (V_\mathbf{b}[\alpha], E_\mathbf{b}[\alpha])$, is the sub-graph of $G_\mathbf{b}$ with the variables with non-zero power in $\mathbf{b}^\alpha$ as its vertex set.
    \item Let \text{Comps}$\left(G_\mathbf{b}\right)$ denote the \textit{set of vertices of connected components of $G_\mathbf{b}$}.
    \item For any component $c\in\text{Comps}(G_\mathbf{b})$, let $\alpha[c]\in\mathbb{N}^{n_b}$ denote the multi-index s.t. the $i_{th}$ element of $\alpha[c]$ equals $\alpha_i$ if its corresponding $v\in V_\mathbf{b}$ is in $c$ and zero otherwise.
\end{itemize}
\end{definition}
\begin{proposition}
\normalfont Given a vector of variables $\mathbf{b}$ with dependence graph $G_\mathbf{b}$ and given a multi-index $\alpha\in\mathbb{N}^{n_b}$:
\begin{align}
    \mathbb{E}[\mathbf{b}^\alpha] = \prod_{c\in\text{Comps}(G_\mathbf{b}[\alpha])}\mathbb{E}\left[\mathbf{b}^{\alpha[c]}\right]
\end{align}\label{prop:moment_to_moment}
\end{proposition}
\begin{proof}
First, note that $\sum_{c\in\text{Comps}(G_\mathbf{b}[\alpha])}\alpha[c] = \alpha$, so we have that:
\begin{align}
    \mathbb{E}[\mathbf{b}^\alpha] = \mathbb{E}\left[\prod_{c\in\text{Comps}(G_\mathbf{b}[\alpha])}\mathbf{b}^{\alpha[c]}\right]
\end{align}
Letting $c\in\text{Comps}(G_\mathbf{b}[\alpha])$, every node in $c$ is independent of every node in $V_\mathbf{b}[\alpha]/c$ as the definition of connected components does not allow for the existence of edges from $c$ to $V_\mathbf{b}[\alpha]/c$. Thus, the product above can be moved outside of the expectation operator.
\end{proof}
\begin{algorithm}
\caption{TreeRing}\label{alg:create_update_relations}
\begin{algorithmic}[1]
\Procedure{Expand}{$\xi$, $\mathbf{b}$, $g$, $G_\mathbf{b}$, $\mathcal{Z}$, $\mathcal{F}$}
    \State $\mathcal{A}, C_\mathcal{A} \leftarrow$ Represent $\mathbf{b}^\xi$ as a polynomial in $\mathbf{b}$ using $g$
    \State Add $\xi$ to $\mathcal{Z}$ and add $\mathcal{A}, \mathcal{C}_\mathcal{A}$ to $\mathcal{F}$
    \ForAll{$\alpha\in\mathcal{A}$}
        \ForAll{$c\in\text{Comps}(G_{\mathbf{b}}[\alpha])$}
            \If{$\alpha[c]\notin\mathcal{Z}$}
                \State Expand($\alpha[c]$, $\mathbf{b}$, $g$, $G_\mathbf{b}$, $\mathcal{Z}$, $\mathcal{F}$)
            \EndIf
        \EndFor
    \EndFor
\EndProcedure
\end{algorithmic}
\end{algorithm}
\subsection{Solving the Motivating Problem with TreeRing}
We return to the motivating problem of deriving an update relation for $\mathbb{E}[x_ty_t]$ by using \textit{TreeRing}. We begin by transforming (\ref{eq:dynamics}) into the polynomial system (\ref{eq:dynamics_poly}) by making the substitutions 

The set of indeterminants, without the time index, is chosen to be:
\begin{align}
    V_\mathbf{b} = \{x, y, v, c, s, w_{v}, s_{w}, c_w\}
\end{align}
In the appendix, we show how the moments of all of the indeterminant variables except $x$ and $y$ can be computed to arbitrary order, so we initialize $\mathcal{Z}$ to contain the multi-indices corresponding to, for example, $c^n$, up to some large $n\in\N$. The edge set for this system under our assumptions is:
\begin{align}
    E_{\mathbf{b}} = \{&(x, y), (x, v), (y, v), (x, s), (x, c), (y, s), (y, c)\}
\end{align}
After running algorithm (\ref{alg:create_update_relations}), we arrive at the following $\mathcal{Z}(\mathbf{b})$:
\begin{align}
        \mathcal{Z}(\mathbf{b}) = \{xy, xs, ys, xc, yc, xvs, xvc, yvs, yvc\}
\end{align}
The result is that if we derive the dynamics for each moment in $\mathcal{Z}(\mathbf{b})$, that set of equations can be used to recursively compute the moments $\mathcal{Z}(\mathbf{b})$ at each time step using only moments of the initial state distribution and moments of $s_w$ and $c_w$.

\section{Experiments}\label{sec:experiments}
In this section, we demonstrate the performance of our system through two learning-based predictors that predict stochastic position and control for the target agent. For each predictor, we describe its network architecture and training procedure, before presenting risk assessment results. All computations were performed on a desktop with an Intel Core i9-7980XE CPU at 2.60 GHz. All Monte Carlo (MC) methods are implemented with vectorized NumPy operations to have a realistic assessment of run times for naive MC.

\subsection{GMM Position Predictor}
\subsubsection{Model Description}
To obtain probabilistic trajectory distributions of agents, we trained a simple DNN to generate Gaussian mixture model parameters for $\mathbf{x}_{1:T}$ over a fixed horizon $T = 30$ given a sequence of observed positions over $20$ time steps. Although our framework works with any prediction model that outputs GMM parameters for positions, we aim to generate accurate and realistic predictions by selecting an encoder-decoder-based predictor that utilizes long short-term memory (LSTM) units because of the recent success of recurrent neural networks in trajectory prediction on different benchmarks \cite{alahi2016social,huang2019diversity,deo2018multi,cui2019deep}.

\begin{figure}[!b]
    \centering
    \vspace*{-5mm}
    \includegraphics[width=1.0\linewidth]{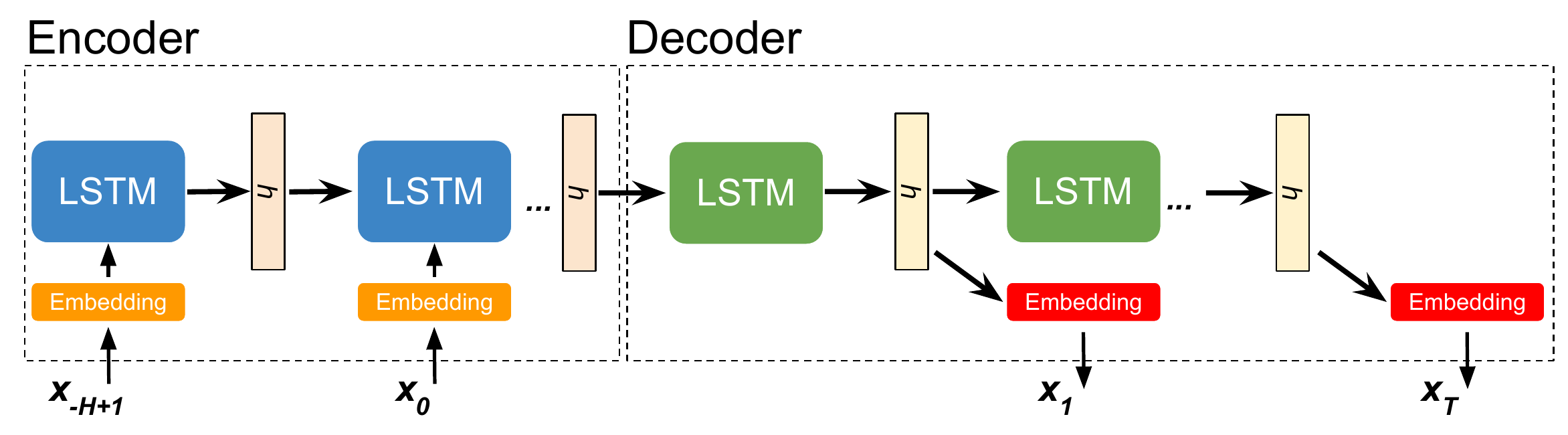}
    \caption{Architecture diagram of GMM position predictor, including an LSTM-based encoder and an LSTM-based decoder. We introduce extra embedding layers to process the observed positions in the encoder and process the predicted hidden states before generating predicted parameters in the decoder.}
    \label{fig:gmm_predictor}
\end{figure}

As shown in Figure~\ref{fig:gmm_predictor}, the encoder is a sequence of LSTM units taking observed trajectories of the target agent as input and outputs a latent vector encoding agent hidden state. The decoder is also a sequence of LSTM units that takes the latent vector and generates a set of GMM position parameters from each LSTM unit. For simplicity, we used three component mixture models. For each component, we generate a mean position vector and a covariance matrix representing uncertainties of predictions. The model is trained and validated on a subset of the Argoverse dataset \cite{chang2019argoverse}.

\subsubsection{Experiments}
On a dataset of 500 scenarios similar to that shown in Figure~\ref{fig:risk_assess_example}, predictions were made and the risk was evaluated along a predefined trajectory for the ego vehicle. To evaluate QFMVG's, we tested both the methods of Imhof and Liu-Tang-Zhang. The methods proposed are much faster than naive Monte Carlo with far lower error.  The method of Imhof with an error tolerance of $10^{-10}$ was used as ground truth \cite{imhof1961computing}. Only $170$ scenarios were used for error computation as results from scenarios with computed ground truth errors within tolerances (i.e: $10^{-10}$) were neglected for error computation. We note that the method of Liu-Tang-Zhang empirically produces results with very small errors while being several times faster than the method of Imhof, which may prove useful in certain contexts.
\begin{table}[]
\caption{Results from evaluating risks in 500 scenarios. Mean time is for evaluating a thirty time step three mode GMM prediction. Errors correspond to the time step with the maximum error.}
\begin{tabular}{|l|l|l|l|}
\hline
\textbf{Method} & \textbf{\begin{tabular}[c]{@{}l@{}}Mean Time \\ (ms)\end{tabular}} & \textbf{\begin{tabular}[c]{@{}l@{}}Mean Max. \\Absolute Error\end{tabular}} & \textbf{\begin{tabular}[c]{@{}l@{}}Mean Max. \\Relative Error\end{tabular}} \\ \hline
Imhof           & $91.21$                   & $0.0$                                                                               & $0.0$                                                                               \\ \hline
Liu-Tang-Zhang  & $26.67$                   & $2.7\times 10^{-6}$                                                                          & $2.3\times 10^{-4}$                                                                           \\ \hline
MC $10^4$ & 106.9                   & $6.7\times 10^{-4}$                                                                           & 0.38                                                                              \\ \hline
MC $5\times 10^4$ & 422.5                   & $2.7\times 10^{-4}$                                                                           & 0.13                                                                              \\ \hline
MC $10^5$ & 1329                    & $1.9\times 10^{-4}$                                                                           & 0.12                                                                              \\ \hline
\end{tabular}
\vspace*{-5mm}
\end{table}

\subsection{GMM Control Predictor}
\subsubsection{Model Description}
We use a similar DNN as the GMM position predictor, but the output becomes instead a set of GMM parameters for control signals defined in \eqref{eq:dynamics}. Again, we assume a fixed number of three components in the mixture model for each control signal for the sake of simplicity. During model training, we obtain the ground truth control sequence by differentiating future positions of the target vehicle. Instead of using the Argoverse dataset, which has noisy differentiated control data due to noise in the perception system used to collect the dataset, we use our own data collected from a naturalistic driving simulator called CARLA \cite{Dosovitskiy17} that provides accurate ground truth control values. The model is trained and validated on 10k samples collected in CARLA.
\begin{figure}[!b]
    \centering
    \vspace*{-5mm}
    \includegraphics[width=0.9\linewidth]{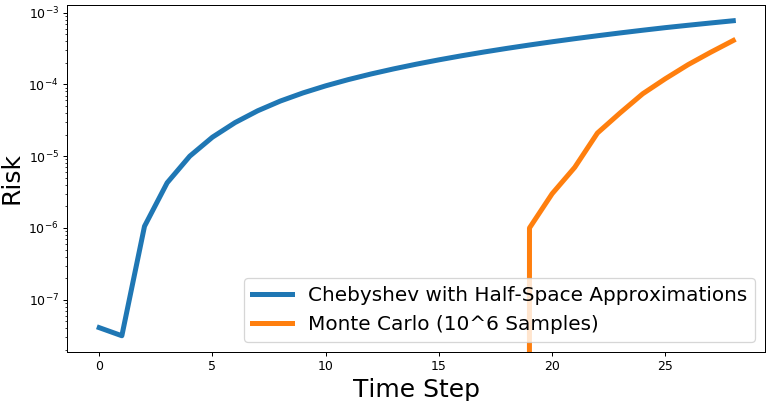}
    \caption{Risk estimates across time for an example scenario using random variables from the GMM control predictor.}
    \label{fig:uncertain_control_chebyshev}
\end{figure}
\subsubsection{TreeRing + Chebyshev Experiments} When deriving expressions for position moments up to order four, \textit{TreeRing} returns 92 polynomial expressions while only 11 are needed to propagate the mean vector and covariance matrix. As these expressions currently require manual transcription into code, a process that is prone to human error, the half-space approximation method with $12$ half-spaces was tested as it requires fewer expressions. The initial state of the agent vehicles was assumed to be known and deterministic. Random variables for control from the DNN and expressions from $\textit{TreeRing}$ were then used to compute the mean and covariance matrix of position at each time step. Over 50 scenarios, the mean time to evaluate the risk for a given trajectory for the Chebyshev method was 80ms while the Monte Carlo method with $10^6$ samples took 140 seconds. The average worst-case conservatism of the Chebyshev risk estimate for a given time step along a trajectory was $0.012$ (assuming the Monte Carlo results represent ground truth). Figure~\ref{fig:uncertain_control_chebyshev} shows the risk for both methods.

\subsubsection{Comparing SOS + Chebyshev}
Experiments were run to test and compare the Chebyshev and SOS methods described in (\ref{chebyshev_quad_form}) and (\ref{sos_subsection}). For this experiment, higher order moments were obtained by automatic differentiation of the MVG moment generating function and the resulting moments of $Q(\mathbf{x}_t)-1$ were normalized. YALMIP was used to transcribe the SOS programs into Semidefinite programs, and SeDuMi was used to solve the resulting semidefinite programs \cite{Lofberg2004, sturm1999using}. We observe that 1) Chebyshev bound produces nearly the same result as the second order SOS program and 2) the SOS program with higher order moments can yield significantly better bounds, especially in the tails. The solve times for each time step only marginally increased for the higher order SOS programs; the mean solve times were 42, 44, and 49 ms for the second, fourth, and sixth order SOS formulations respectively. While these solve times are much better than those often encountered with SOS programs, further advances in performance are needed for this to be used online.
\begin{figure}[!t]
    \centering
    \includegraphics[width=\linewidth]{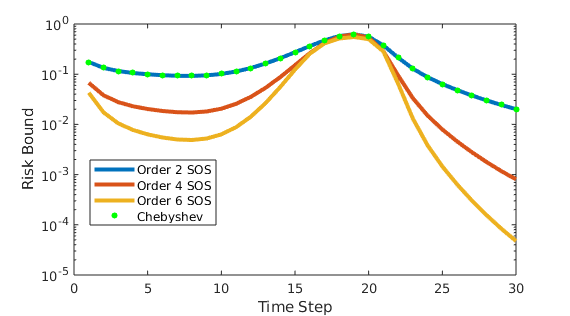}
    \caption{Risk bounds computed with the SOS formulation from section \ref{sos_subsection} compared with Chebyshev's inequality without half-space approximations.}\label{fig:my_label}
    \vspace*{-5mm}
\end{figure}
\section{Conclusions}
Our experimental results with the methods for GMM position models and the Chebyshev method for non-Gaussian position and control models suggest that high performance implementations would be immediately practical for use in online applications. Future work should incorporate risk assessment into motion planning algorithms, but we note that it may be easily incorporated into standard algorithms with ``collision check" primitives such as RRTs and PRMs \cite{rarnop}. The SOS method does significantly improve upon the risk bounds from the Chebyshev method at the cost of additional computation; future research should further develop methods to make the SOS step offline to improve runtimes for online applications \cite{jasour2018moment}. Future work in both prediction and risk assessment should also work towards relaxing assumptions such as time independence.

\section*{Acknowledgements}
This work was supported in part by Boeing grant MIT-BA-GTA-1 and by the Masdar Institute grant 6938857. Allen Wang was supported in part by a NSF Graduate Research Fellowship.

\section{Appendix}
\subsection{Moments and Characteristic Functions of Mixture Models}
Letting $f_X$ denote the pdf of a $K$ component mixture model, $X$, with component pdfs $f_{X_i}, \forall i\in[K]$ and letting $f_Z$ denote the pdf of the $K$ category Multinoulli, by definition $f_X(x) = \sum_{i=1}^m f_{X_i}(x) f_Z(i)$.
For any measurable function, $g$, by interchanging the order of integration and summation
\begin{align}
    \Exp[g(X)] &= \int g(x) f_X(x) dx\\
    &= \sum_{i=1}^K f_Z(i)\int g(x) f_{X_i}(x) dx
\end{align}
By letting $g(X) = X^n$ or $g(X) = e^{itX}$, we have that the moments and characteristic function of $X$ can both be computed as the weighted sum of those of their components.
\subsection{Moments of Trigonometric Variables}
In this section, we show how moments of the form $\mathbb{E}[\cos^n(X)]$, $\mathbb{E}[\sin^n(X)]$, and $\mathbb{E}[\cos^m(X)\sin^n(X)]$ can be computed in terms of the characteristic function of the random variable $X$, $\Phi_{X}$. We begin by applying Euler's Identity to the definition of the characteristic function:
\begin{align}
\begin{split}
    \Phi_{X}(t) &= \mathbb{E}[e^{itX}]\\
    &= \mathbb{E}[\cos(tX)] + i\mathbb{E}[\sin(tX)]
\end{split}
\end{align}
Thus, we have that $\mathbb{E}[\cos(tX)] = \text{Re}(\Phi_{X}(t))$ and $\mathbb{E}[\sin(tX)] = \text{Im}(\Phi_{X}(t))$. This immediately gives us the ability to compute the first moments of our trigonometric random variables. For higher moments, the trigonometric power formulas can be used to express quantities of the form $\cos^n(X)$ as the sum of quantities of the form $\cos(mX)$ where $m\in\N$ and similarly for $\sin^n(X)$ \cite{zwillinger2002crc}. Thus, higher moments of $\sin(X)$ and $\cos(X)$ can be computed using $\Phi_X(t)$.
Moments of the form:
\begin{align}\label{eq:higher_cos_sin}
    \mathbb{E}[\cos^m(X)\sin^n(X)]
\end{align}
can also ultimately be computed in terms of $\Phi_X(t)$. This can be seen if we make the substitutions $\cos(X) = \frac{1}{2}(e^{ix} + e^{-ix})$ and $\sin(X) = \frac{1}{2i}(e^{ix} - e^{-ix})$, then (\ref{eq:higher_cos_sin}) can be expressed as:
\begin{align}
    \mathbb{E}\left[\frac{1}{i^n2^{m + n}}(e^{iX} + e^{-iX})^m(e^{iX} - e^{-iX})^n\right]
\end{align}
By applying the binomial theorem to  both expressions in parentheses, and multipying the resulting expressions, we find the entire expression in the expectation operator can be expressed as a polynomial in $e^{iX}$ and $e^{-iX}$. Thus, the entire expression can be written as the sum of terms of the form $\mathbb{E}[e^{itX}]$ for $t\in\mathbb{Z}$ which is in the definition of $\Phi_X(t)$. In practice, computer algebra packages were used to derive these expressions when needed.
\subsection{Sum of Independent Random Variables}
In the case that our random variable X is a sum of independent random variables $Y_i$ and a constant $c$, $X=c+\sum_{i\in[n]} Y_i$, the characteristic function $\Phi_X$ can be expressed as such:
\begin{align}
    \Phi_X(t)=e^{itc}\prod_{i\in[n]}\Phi_{Y_i}(t)
\end{align}
\bibliographystyle{IEEEtran}
\bibliography{references} 

\end{document}